\documentclass[journal]{IEEEtran}
\ifCLASSINFOpdf
\else
\fi
\usepackage{url}
\usepackage{amssymb}

\usepackage{bm}
\usepackage[hidelinks]{hyperref}
\usepackage{threeparttable}
\usepackage{stfloats}
\usepackage{algorithm}
\usepackage{algorithmic, amsmath, graphicx}
\usepackage{amsmath,amsthm}
\usepackage{mathrsfs}
\usepackage{booktabs}
\usepackage{subfigure}
\usepackage{color}
\newtheorem{lemma}{Lemma}
\newtheorem{defn}{Definition}

\long\def\comment#1{}

\begin{document}
\title{Rectified Decision Trees: Exploring the Landscape of Interpretable and Effective Machine Learning}

\author{Yiming Li$^*$, Jiawang Bai$^*$, Jiawei Li, Xue Yang, Yong Jiang, and Shu-Tao Xia
\thanks{$^*$ indicates equal contribution.}
\thanks{All authors are with Tsinghua Shenzhen International Graduate School,  Tsinghua University, Guangdong 518055, China.}
\thanks{Corresponding author: Xue Yang (\href{mailto:xueyang.swjtu@gmail.com}{xueyang.swjtu@gmail.com}) and Shu-Tao Xia (\href{mailto:xiast@sz.tsinghua.edu.cn}{xiast@sz.tsinghua.edu.cn}).}}


\markboth{Preprint, under review}%
{Shell \MakeLowercase{\textit{et al.}}: Bare Demo of IEEEtran.cls for IEEE Journals}



\maketitle

\begin{abstract}
Interpretability and effectiveness are two essential and indispensable requirements for adopting machine learning methods in reality. In this paper, we propose a knowledge distillation based decision trees extension, dubbed rectified decision trees (ReDT), to explore the possibility of fulfilling those requirements simultaneously. Specifically, we extend the splitting criteria and the ending condition of the standard decision trees, which allows training with soft labels while preserving the deterministic splitting paths. We then train the ReDT based on the soft label distilled from a well-trained teacher model through a novel jackknife-based method. Accordingly, ReDT preserves the excellent interpretable nature of the decision trees while having a relatively good performance. The effectiveness of adopting soft labels instead of hard ones is also analyzed empirically and theoretically. Surprisingly, experiments indicate that the introduction of soft labels also reduces the model size compared with the standard decision trees from the aspect of the total nodes and rules, which is an unexpected gift from the `dark knowledge' distilled from the teacher model. 
\end{abstract}


% \fi
\IEEEpeerreviewmaketitle

\section{Introduction}
\IEEEPARstart{I}{nterpretability}, which indicates that the prediction process is interpretable, and effectiveness are two essential and indispensable criteria to evaluate whether a machine learning algorithm can be safely adopted in reality. However, a dilemma exists for existing methods. Specifically, the models with simple structure ($e.g.$, the linear regression \cite{galton1886regression}, logistic regression \cite{hosmer2013applied}, and decision trees \cite{breiman1984classification}) are naturally interpretable, but the corresponding performance is too poor to be practically applied. In contrast, the complex methods, such as forest-based algorithms and deep neural networks (DNNs), can provide attractive performance \cite{grbovic2018real, cui2018detection, li2018deep}, but their interpretability is unsatisfied. This drawback is unacceptable in many realms, especially for medical diagnosis or risk assessment.

To address this problem, the most straightforward and wildly used approach is to improve the interpretability of those complex methods. As a representative, the interpretability of forest-based algorithms ($e.g.$, random forests \cite{breiman2001} and gradient boosting decision trees (GBDT) \cite{friedman2001}) and deep neural networks (DNNs) attract the most attention. In the aspect of the forest-based algorithms, the key improvement idea is to simplify the model structure. For example, Breiman and Shang \cite{breiman1996} first proposed to simplify random forests through a single sub-tree. In \cite{meinshausen2010}, node harvest was proposed to simplify tree ensembles by using the shallow parts of the trees. Considering the simplification of tree ensembles as a selection problem, a Bayesian-based method was proposed in \cite{hara2018}. In the meanwhile, the interpretability research of DNN can be divided into three main aspects: visualizing the representations in intermediate layers of DNN \cite{zeiler2014,zhou2018}, representation diagnosis \cite{yosinski2014,zhang2018} and building explainable DNNs \cite{chen2016b,zhang2018b}. Unfortunately, although all mentioned attempts have made certain improvements, they still suffer from either relatively poor performance or insufficient interpretability.

To balance the interpretability and the performance of the model, in this paper, we explore another angle by combining the interpretable nature of decision trees (DT) and the excellent learning ability of complex models. We propose a knowledge distillation based decision trees rectification, dubbed rectified decision trees (ReDT), to explore the possibility of fulfilling the effectiveness and interpretability simultaneously. The critical difference between ReDT and DT lies in the use of soft labels in the process of tree splitting. Specifically, the label is mainly involved in two processes of the tree's splitting: \textbf{(1)} calculating the change of impurity and \textbf{(2)} determining whether the \emph{pure stopping criterion} ($i.e.$, all samples are from the same class in the node) is satisfied. Compared to using hard labels in DT, we propose to use soft labels in two aforementioned training processes in ReDT. The main contributions of this paper can be summarized as follows: 
\begin{itemize}
\item \textbf{How to obtain soft labels}. Inspired by the knowledge distillation \cite{hinton2015}, we obtain soft labels based on the predicted logits generated by a well-trained model (dubbed \emph{teacher model}). Accordingly, the proposed ReDT can utilize the excellent learning ability of the teacher model through the `dark knowledge' contained in soft labels. It is worth noting that the teacher model used in ReDT can be DNN, tree ensemble or any other classification algorithms, while the backpropagation \cite{rumelhart1985learning} of the student model is not necessarily compared with the classical knowledge distillation. Therefore the proposed method can be regarded as an attempt of a new knowledge distillation approach. Besides, different from the temperature function used in the classical knowledge distillation \cite{hinton2015,frosst2017}, we propose a novel \emph{jackknife-based method} to reduce the adverse effects of randomness and overfitting when generates soft labels.
\item \textbf{How to use soft labels}. We extend the standard DT to allow training with soft labels while preserving the deterministic splitting paths. Firstly, in ReDT, the proportion of the samples is determined based on the average of all soft labels in the node when calculating the change of impurity. Secondly, we use the \emph{pseudo label}, which is the class with the highest probability within the soft label, to determine whether all samples in the node are from the same class. We use the pseudo label instead of the soft label itself since it is almost impossible for all soft labels within a node to be the same. In this way, the stopping criterion can be better met, therefore the model has better convergence.  
\item \textbf{A novel perspective of interpretable machine learning}. Our research illustrates how to \emph{squeeze} the potential of existing interpretable models by utilizing the learning ability of complex models, which provides a new angel towards interpretable and effective machine learning.
\end{itemize}

%The essence of the aforementioned problem is actually how to balance the interpretability and the performance of the model. In this paper, we explore another angle by combining the interpretable nature of DT and the excellent learning ability of complex models. Specifically, we propose a knowledge distillation based decision trees rectification, dubbed rectified decision trees (ReDT), to explore the possibility of fulfilling the effectiveness and interpretability simultaneously. The critical difference between ReDT and DT lies in the use of soft labels in the process of tree splitting. Specifically, the label is mainly involved in two processes of the tree's splitting: \textbf{(1)} calculating the change of impurity and \textbf{(2)} determining whether the \emph{pure stopping criterion} ($i.e.$, all samples are from the same class in the node) is satisfied. The use of hard labels is needed in DT, while we introduce soft labels into these processes. 

%However, the \emph{interpretability}, which indicates that the prediction generated by the model can be interpreted, of those methods is unsatisfied. This drawback is unacceptable in many realms, such as in medical diagnosis or risk assessment. 

\section{Related Work}
%Interpretability and effectiveness are two essential and indispensable requirements for safely adopting machine learning methods in reality. Except for a few models with simple structure, such as the linear regression, logistic regression, and decision trees, that are naturally interpretable, the prediction of most existing machine learning methods cannot be explained. In contrast, those naturally interpretable methods have a relatively bad performance, which hinders their usage. Accordingly, how to balance the interpretability and effectiveness is of great significance and attracts lots of attention. This research main contains two directions: one the one hand, it directly focuses on improving the interpretability of effective models; On the other hand, it indirectly improves interpretability by compressing the model structure. 

\subsection{Interpretable and Effective Machine Learning}
How to balance the interpretability and effectiveness is of great significance and attracts lots of attention. Existing works mainly focus on how to improve the interpretability of effective methods. The research of this area can be traced back to 1996, where Breiman and Shang proposed to simplify the complex random forests through selecting a single sub-tree \cite{breiman1996}. At that time, random forest is one of the most effective learning algorithms. After that, by linking the interpretability of the random forest to the depth of trees in the forest, Meinshausen proposed to simplify tree ensembles by using the shallow parts of the trees \cite{meinshausen2010}. Besides, a Bayesian-based approach is also provided to select main rules contained in tree ensembles for model simplification to enhance interpretability. Recently, due to the widespread success of DNNs, their interpretability has attracted most of the attention. For example, \cite{zeiler2014} and \cite{zhou2018} tried to explain the prediction of DNNs through visualizing representations in intermediate layers; \cite{zhang2018b} built explainable DNNs through specifying the function of different DNN components; \cite{yosinski2014} and \cite{zhang2018} were proposed to conduct the representation diagnosis. However, all mentioned works are either suffer from either relatively poor performance or insufficient interpretability. How to construct an interpretable model with great performance is still an important open question.

\subsection{Knowledge Distillation}
Knowledge distillation (KD) is widely adopted in model compression, which compresses a big model ($e.g.$, DNNs) into a smaller one, where the compressed model preserves the function learned by the complex model. KD can be considered to be motivated by the transfer learning, which transfers knowledge learned by a large teacher model into a smaller student model by learning the class distribution \cite{hinton2015}.

The first work was proposed by Hinton et al., where a teacher-student learning paradigm was adopted \cite{hinton2015}. Specifically, softened logits obtained from the teacher DNN based on the temperature function is used to teach a small student DNN. After that, the focus of past work has been either on \textbf{(1)} improving the performance or \textbf{(2)} finding new applications for the idea. In general, to improve the performance, prior works usually introduced additional loss terms on intermediate feature maps of the student to bring them closer to those of the teacher \cite{yim2017, liu2019structured}. Recently, KD was also adopted in the study of interpretability. In \cite{frosst2017}, they used a trained DNN to create a more explainable model in the form of soft decision trees. However, soft decision trees \cite{irsoy2012} is more like a tree-shape DNN and its decision path is probabilistic rather than deterministic, therefore its interpretability is much weaker than that of standard decision trees. Besides, the proposed KD method in \cite{frosst2017} requires the backpropagation of the student model, therefore it can not be adopted directly to the standard decision trees.

\section{Proposed Method}
In this section, we first briefly review the standard decision trees and knowledge distillation in Section \ref{Prelim}. Then we discuss two processes involved in the proposed method, including the generation of soft labels and training based on obtained soft labels. Specifically, Section \ref{DK} introduces how to obtain the soft label, and Section \ref{TC} discusses the specific training and prediction process of the proposed rectified decision trees (ReDT). An analysis, which demonstrates why soft labels can reach a better performance than hard labels is provided in Section \ref{EA}.

\subsection{Preliminaries} \label{Prelim}
Let $\mathcal{D}_n$ represent a dataset consisting of $n$ $i.i.d.$ observations. Each observation has the form $(\bm{x},y)$, where $\bm{x} \in \mathbb{R}^D$ represents the $D$-dimensional features and $y \in \{1, \cdots, K\}$ is the corresponding label of $\bm{x}$.

\vspace{0.3em}
\noindent \textbf{Impurity Calculation in Standard DT. } The impurity $I$ of node $\mathcal{N}$ is calculated based on the class proportion of samples within the node. In other words, $I=T(P)$, where $T(\cdot)$ is the impurity criterion ($e.g.$, Shannon entropy), $P=(p_1, \cdots, p_K)$ such that $p_i = \frac{1}{|\mathcal{N}|} \sum_{(\bm{x},y) \in \mathcal{N} } \mathbb{I} \{y=i\}$, and $|\mathcal{N}|$ denotes the number of sample in $\mathcal{N}$. For example, if Shannon entropy is adopted as the criterion, then $I = -\sum_{i=1}^{K} p_i\log{p_i}.$ 

\vspace{0.3em}
\noindent \textbf{Ending Conditions in Standard DT. }
\begin{itemize}
    \item \emph{Pure Ending Condition. } For a splitting candidate node, if all samples within it are from the same category, the  node is considered to be \emph{pure} and the splitting will be terminated, $i.e.$, $\mathbb{I}\{\exists i \in \{1, \cdots K\},~ s.t.\  p_i =1\}$, where $p_i$ indicates the proportion of the samples with $i$-th category.
    \item \emph{Early Stopping Condition. } Except for the aforementioned pure ending condition, early stopping is also introduced in DT to alleviate overfitting. Specifically, if the number of samples within the splitting candidate node is less than a given minimum leaf size $k$, the splitting will also be terminated.
\end{itemize}

\vspace{0.3em}
\noindent \textbf{Prediction Process in Standard DT. } After a series of decisions, $\bm{x}$ will eventually fall into a leaf node, assuming that node is $\mathcal{V}$. The predicted label is also determine by $P=(p_1, \cdots, p_K)$, where $p_i = \frac{1}{|\mathcal{V}|} \sum_{(\bm{x},y) \in \mathcal{V}} \mathbb{I} \{y=i\}$.

\vspace{0.3em}
\noindent \textbf{Temperature-based Distillation. } This method was first proposed in the compression of deep neural networks \cite{hinton2015}, which generates a softened version of the final output of a teacher DNN. The distilled knowledge will then be used to train a smaller \emph{student DNN} to transfer the generalization ability of the teacher model.  This method is especially effective when the distilled sample is already used for the training of the teacher model since the predicted logit of training samples is often close to the one-hot vector.  

Specifically, given a well-trained teacher DNN $f_{\bm{\theta}}(\cdot)$, and a sample $\bm{x}$, let $f_{\bm{\theta}}(\bm{x}) = (p_1, \cdots, p_K)$ indicates the predicted logits of $\bm{x}$. The distilled knowledge $L_{soften}(\bm{x})$ ($i.e.$, soften logits) is obtained through 
\begin{equation}
    L_{soften}(\bm{x}) = \text{softmax}(f_{\bm{\theta}}(\cdot)/T),
\end{equation}
where $T$ is a non-negative given hyper-parameter (dubbed \emph{temperature}). The higher the temperature, the soften the distilled knowledge. 

In what follows, we give the details about the proposed ReDT.
\subsection{The Generation of Soft Labels}\label{DK}
\begin{defn} [Hard Label and Soft Label]\label{def:labels}
\ 

\begin{itemize}
    \item The hard label $\bm{y}_{hard}$ indicates the one-hot representation of $y$, $i.e.$, $\bm{y}_{hard} = \left(\bm{y}_{hard}^{(1)}, \cdots, \bm{y}_{hard}^{(K)}\right) \in \{0,1\}^K$, and the $c$-th component of the hard label $\bm{y}_{hard}^{(c)}=1$ means that the sample $\bm{x}$ belongs to the class $c$.
    \item Given a well-trained classifier with parameter $\bm{\theta}$, the soft label of $\bm{x}$, $i.e.$, $\bm{y}_{soft}$,  is the predicted logits of $\bm{x}$, $i.e.$, $\bm{y}_{soft} = f_{\bm{\theta}}(\bm{x})$, where $f_{\bm{\theta}}(\bm{x})$ is the generated logits of $\bm{x}$. 
\end{itemize}

\end{defn}

From the perspective of statistical machine learning, the training of the model can be regarded as an approximation process toward the (unknown) latent distribution $\mathbb{P}_{X \times Y}$. Intuitively, it is extremely difficult to recover the ground-truth distribution from $\{(\bm{x}, y)\}$ directly, since label $y$ contains less distribution information. In contrast, the output logits ($i.e.$, the output probability vector) of a well-trained classifier consists of a significant amount of useful information of the distribution compared with the original label $y$ itself. Inspired by this idea, we propose to use the \emph{soft label} instead of the \emph{hard label}, as defined in Definition \ref{def:labels}, in the training process of student model ($i.e.$, ReDT). The obtained soft labels can be regarded as the distilled knowledge from a well-trained \emph{teacher model}, which utilizes its excellent learning ability. This idea is also partly supported by \cite{hinton2015} where a softened version of the final output of a teacher network is used to teach information to a small student network.

\begin{algorithm}[!ht]
   \caption{$M$-folds Jackknife-based Distillation Method.}
   \label{alg: Jackknife} 
\begin{algorithmic}[1]
\REQUIRE ~~\\
Training dataset $\mathcal{D}=\{(\bm{x},y)\}$;
\ENSURE ~~\\ 
Training dataset $\mathcal{D}'=\{(\bm{x},y_{soft})\}$ with soft labels;
\STATE Randomly divide the training set $\mathcal{D}$ into $M$ equal and disjoint subsets $\mathcal{D}_1, \cdots, \mathcal{D}_M$.
    \FOR{$i=1 \cdots M$}        
        \STATE Train teacher model $f_{\bm{\theta}}(\cdot)$ with samples $\mathcal{D} - \mathcal{D}_i$
        \STATE Generate samples with soft labels through 
        $$\mathcal{D}_i' = \{(\bm{x}, y_{soft})|y_{soft} = f_{\bm{\theta}}(\bm{x}), (\bm{x}, y) \in \mathcal{D}_i \}$$
    \ENDFOR
\STATE $\mathcal{D}' = \cup_{i=1}^{M} \mathcal{D}_i'$    
\STATE {\bfseries Return:} $\mathcal{D}'$
\end{algorithmic}
\end{algorithm}

The remaining problem is how to obtain soft labels for training ReDT. Once a well-trained teacher model $f_{\bm{\theta}}(\cdot)$ is given, the generation of the soft label is straightforward. However, we only have training samples in most cases. Under such circumstance, the most straightforward idea is to train the teacher model using all training samples and then generate soft labels. Unfortunately, the soft label obtained through such process has relatively poor quality due to the adverse effects of overfitting, $i.e.$, the generated soft label is closed to its hard version. To address this problem, we propose a novel \emph{jackknife-based distillation method}, as shown in Algorithm \ref{alg: Jackknife}. To reduce the effect of randomness, we run this process several times and adopt their average to be the final soft label. Note that this problem does not exist in previous knowledge distillation, since their training of student model and teacher model is carried out simultaneously rather than strictly one after another, thanks to these models can both be trained through back propagation. %The calculation process of the proposed method is easy to be parallelized, therefore it with high efficiency.

\subsection{Rectified Decision Trees (ReDT)}\label{TC}
After generating soft labels, we can adopt them to train the model. In this section, we propose a novel decision trees extension, dubbed rectified decision trees (ReDT), which allows training with soft labels while preserving deterministic splitting paths. Specifically, the difference between ReDT and DT lies in two aspects, including \textbf{(1)} \emph{splitting criteria} and \textbf{(2)} \emph{pure ending condition}.

\vspace{0.3em}
\noindent \textbf{Splitting Criteria.}
Recall that the impurity decrease caused by splitting point $v$ is denoted by
\begin{equation}\label{1}
    I(v) = T(\mathcal{D}) - \frac{|\mathcal{D}_l|}{|\mathcal{D}|}T(\mathcal{D}_l)
    -\frac{|\mathcal{D}_r|}{|\mathcal{D}|}T(\mathcal{D}_r),
\end{equation}
where $\mathcal{D}_{l}, \mathcal{D}_{r}$ are two children sets from $\mathcal{D}$ splitting at $v$, $T(\cdot)$ is the impurity criterion ($e.g.$, Shannon entropy or Gini index). In ReDT, the probability $p_i$, which implies the proportion of the samples with $i$-th category, used in calculating the impurity decrease is now calculated through soft labels. 
%%%%%%%%
\comment{
Specifically, since each sample uses the soft label instead of the hard one, we calculate the average of the soft label of all training samples in the node and finally obtain a $K$-dimensional vector. At this time, $p_i$ is redetermined as the value of the $i$-th dimension of that vector. In other words
}
Specifically, let $\bm{y}_{soft}^{(j)}=\left(y_{1}^{(j)},y_{2}^{(j)},\cdots,y_{K}^{(j)}\right)$ denotes the soft label of $j$-th training sample contained in the node $\mathcal{N}$, then the probability $p_i$ of node $\mathcal{N}$ is calculated as
\begin{equation}
    p_i=\frac{1}{|\mathcal{N}|}\sum_{\left(\bm{x}^{(j)}, \bm{y}_{soft}^{(j)}\right)\in \mathcal{N}} y_{i}^{(i)},
\end{equation}
where $|\mathcal{N}|$ denotes the number of samples in the node $\mathcal{N}$.

\begin{algorithm}[ht]
   \caption{The Training Process of ReDT. }
   \label{alg-dt}
\begin{algorithmic}[1]
\REQUIRE ~~\\
Training dataset $\mathcal{D}=\left\{(\bm{x},\bm{y}_{mixed})\right\}$;\\
Minimum leaf size $k$;
\ENSURE ~~\\ 
The rectified decision trees $T \leftarrow ReDT(\mathcal{D}, k)$;
   \STATE Calculate pseudo-category $y_{pseudo}$ of each sample in $\mathcal{D}$ by Eq. \eqref{pseudo}.
   \STATE Determine whether the node is pure based on whether each sample in $\mathcal{D}$ has the same pseudo-category.
        \IF{sample size  $>k$, and the node is not pure} 
         \STATE Calculate the impurity decrease vector $I$ according to Eq. \eqref{1}.
        \STATE Select the splitting point with maximum impurity reduction criterion. 
        \STATE The training set $\mathcal{D}$ correspondingly is split into two child nodes, called $\mathcal{D}_{l}, \mathcal{D}_{r}$.
        \STATE $T.leftchild \leftarrow ReDT(\mathcal{D}_l, k)$
        \STATE $T.rightchild \leftarrow ReDT(\mathcal{D}_r, k)$
        \ENDIF
\STATE {\bfseries Return:} $T$.
\end{algorithmic}
\end{algorithm}

\vspace{0.3em}
\noindent \textbf{Pure Ending Condition.}
The second alteration is how to define \emph{pure} to determine whether the splitting should be terminated. In the standard decision trees, if all samples in a node have the same category, the node is considered to be pure. However, it is almost impossible for all soft labels within a node to be the same, therefore the traditional pure stopping criterion will never meet. Accordingly, we use the \emph{pseudo label} $y_{pseudo}$ to determine whether the node is pure by evaluating the pseudo label of all samples in the node. The \emph{pseudo label} is the class with the highest probability of the soft label, $i.e.$,

\begin{equation}\label{pseudo}
    y_{pseudo}=\arg \max \bm{y}_{soft}.    
\end{equation}

Besides, instead of using the soft label of samples directly, we use the \emph{mixed label} $\bm{y}_{mixed}$, which is the weighted average of soft label and hard label with weight hyper-parameter $\alpha \in [0, 1]$. That is,
\begin{equation}\label{mixed}
    \bm{y}_{mixed} = \alpha \cdot\bm{y}_{hard} + (1-\alpha) \bm{y}_{soft}.
\end{equation}
The hyper-parameter $\alpha$ plays a role in regulating the proportion of using the soft label. The larger $\alpha$, the smaller the proportion of the soft label in the mixed label. When $\alpha = 1$, the ReDT degrades into the Breiman's decision trees. The purpose of using mixed labels is to consider that the soft label may have a certain degree of error. By adjusting the hyper-parameter $\alpha$, we can obtain the mixed label with sufficient distilled knowledge and relatively high accuracy.

The overall training process is shown in Algorithm \ref{alg-dt}.

\vspace{0.3em}
\noindent \textbf{The Prediction of ReDT. }
Once the ReDT has grown based on mixed labels as described in Algorithm \ref{alg-dt}, the prediction for a newly given sample $\bm{x}$ is similar to the standard decision trees \cite{breiman1984classification}. Specifically, suppose the predicted label and predicted logits of sample $\bm{x}$ is $\hat{y}$ and $\bm{P}=(\hat{p}_1, \cdots, \hat{p}_K)$ respectively. According to a series of decisions, $\bm{x}$ will eventually fall into a leaf node, assuming that node is $\mathcal{V}$. The predicted logits of $\bm{x}$ is the average of mixed labels of all training samples within node $\mathcal{V}$, $i.e.$,
\begin{equation}
    \bm{P}=(\hat{p}_1, \cdots, \hat{p}_K) = \frac{1}{|\mathcal{V}|} \sum_{\left(\bm{x}^{(i)}, \bm{y}_{mixed}^{(i)}\right) \in \mathcal{V}} \bm{y}_{mixed}^{(i)},
\end{equation}
where $|\mathcal{V}|$ denotes the number of samples in the leaf node $\mathcal{V}$.

The predicted label of $\bm{x}$ is the one with highest probability in $\bm{P}$:
\begin{equation}
    \hat{y}= \arg \max_i \hat{p_i}.
\end{equation}

\subsection{Effectiveness Verification of Soft Labels}\label{EA}
The reason why soft labels rather than hard labels should be used can be further verified from the perspective of impurity calculation. 

\begin{lemma}[Integer Partition Lemma]\label{l1}
Suppose there is an integer $N$, which is the sum of $k$ integers $n_1, \cdots, n_k$, $i.e.$, 
$$
N=n_1+n_2+\cdots+n_k.
$$
There are totally $C_{n+k-1}^{k-1}=\frac{(n+k-1)!}{(k-1)!n!}$ possible values for the ordered pair $(n_1, \cdots,n_k)$.
\end{lemma}
\begin{proof}
This problem is equivalent to pick $k-1$ locations randomly from $n+k-1$ locations. The result is trivial based on the basics of number theory. 
\end{proof}

Lemma \ref{l1} indicates that for a $K$-classification problem, if the node $\mathcal{N}$ contains $N$ samples, then the impurity of this node has at most $C_{N+K-1}^{K-1}$ possible values. In other words, compared to use soft labels, the use of hard labels limits the precision of the impurity of the nodes. This limitation has a great adverse effect on the selection of the split point, especially when the number of samples is relatively small.

%From another perspective, the improvement brought by soft labels is since it is tough to recover the distribution of $(X, Y)$ with hard labels directly, especially when the number of samples is relatively small. However, once the relatively correct softened labels are provided, a large amount of information of the distribution is contained in it. The use of this information about the distribution makes the decision surface offset towards the real position compared to when using the hard label.

\begin{table}[ht]
\caption{The description of benchmark datasets.}
\label{table1}
\begin{center}
\begin{sc}
\begin{tabular}{lccc}
\toprule
Data set & Categories & Features & Instances   \\
\midrule
ADULT      & 2  & 14  & 48842 \\
CRX        & 2  & 15  & 690   \\
EEG        & 2  & 15  & 14980 \\
BANK       & 2  & 17  & 45211 \\
GERMAN     & 2  & 20  & 1000  \\
CMC        & 3  & 9   & 1473  \\
CONNECT-4  & 3  & 42  & 67557 \\
LAND-COVER & 9  & 147 & 675   \\
LETTER     & 26 & 15  & 20000 \\
ISOLET     & 26 & 617 & 7797  \\ \hline
MNIST      & 10 & 784   & 60000 \\
\bottomrule
\end{tabular}
\end{sc}
\end{center}
\end{table}

\section{Experiments}
\label{exp}
\subsection{Experimental Configuration}
\noindent \textbf{Benchmark Datasets.}
Following the configuration of decision trees variants \cite{breiman1996, irsoy2012, hara2018}, the benchmark datasets are selected from the UCI repository \cite{UCI} in the evaluation of the ReDT with forest-based teachers. The performance of ReDT with DNN-based teachers is evaluated on the MNIST dataset \cite{lecun1998}. The detailed information about benchmark datasets is shown in Table \ref{table1}.

\vspace{0.3em}
\noindent \textbf{The Settings of Teacher Model.}
In this paper, we evaluate two types of teacher model ($i.e.$, the one generates soft labels), including forest-based teacher and DNN-based teacher. Specifically, we evaluate ReDT with two most widely used forest-based models. including random forests (RF) \cite{breiman2001} and gradient boosting decision trees (GBDT) \cite{friedman2001}. They are the representative of the bagging and boosting method in forest-based teachers, respectively. In the evaluation with DNN-based teachers, we examine a variety of DNN architectures, including MLP \cite{rumelhart1985learning}, LeNet-5 \cite{lecun1998}, and VGG-11 \cite{Simonyan2014}. The MLP has two hidden layers, with 784 and 256 units respectively.

\vspace{0.3em}
\noindent \textbf{Training Setup.} In the experiments with forest-based teachers, five times 5-folds jackknife-based distillation method is used to generate soft labels of training samples. When generating soft labels with DNN-based teachers, the standard temperature-based method is used with temperature $T=4$ as suggested in \cite{hinton2015}. 
The Gini index was used in RF, DT and ReDT as the impurity measure, the minimum leaf size $k=5$ is set for RF, GBDT, DT, and ReDT as suggested in \cite{breiman2001}, and the \emph{pruning} technique is not involved in all methods. The number of trees contained in both RF and GBDT is set to $100$. We train $50$ epochs for all DNNs with Adam optimizer and an initial learning rate of $0.1$. The learning rate is decreased by a factor $10$ at epochs $20$ and $40$, respectively. We determine the value of $\alpha$ by grid search with a step of $0.1$ in the range $[0,1]$, and the effect of $\alpha$ is further evaluated in Section \ref{Abalation}. All forest-based teachers are implemented based on the scikit-learn platform \cite{pedregosa2011}, and the implementation of DNN-based teachers is based on Pytorch deep learning framework \cite{paszke2019pytorch}. 

\noindent \textbf{Evaluation Setup.} The performance of ReDT is evaluated mainly based on the test accuracy and the number of nodes in this paper. Compared with standard decision trees, ReDT with better performance (higher accuracy or smaller nodes) is indicated in boldface. We carry out the experiment ten times to reduce the effect of randomness. Besides, we also conduct Wilcoxons signed-rank test \cite{demvsar2006} to verify whether there exists a difference between the results of the ReDT and those of decision trees at significance level $0.05$. Those that have a statistically significant difference from the decision trees are marked with "$\bullet$".

\begin{table*}[ht]
\caption{Comparison among ReDT with forest-based teachers and other methods. We examine ReDT with two teacher models, including RF and GBDT, dubbed ReDT (RF) and ReDT (GBDT), respectively. $\alpha^{*}$ indicates the average of all best $\alpha$ for each experiment. }
\label{table2}
\begin{center}
\begin{threeparttable}
\begin{tabular}{l|c|ccc|ccc}
\hline
Dataset    & DT      & RF      & {\upshape ReDT (RF)}         & $\alpha^{*}$ (RF) & GBDT    & {\upshape ReDT (GBDT)}       & $\alpha^{*}$ (GBDT) \\ \hline
ADULT      & 81.86\% & 86.54\% & \textbf{86.18\%}$^{\bullet}$ & 0.01              & 86.53\% & \textbf{86.16\%}$^{\bullet}$ & 0.06                \\
CRX        & 80.51\% & 86.14\% & \textbf{85.46\%}$^{\bullet}$ & 0                 & 86.09\% & \textbf{84.40\%}$^{\bullet}$ & 0.11                \\
EEG        & 82.88\% & 81.50\% & \textbf{83.02\%}             & 0.24              & 90.58\% & \textbf{83.01\%}             & 0.52                \\
Bank       & 87.60\% & 90.38\% & \textbf{90.11\%}$^{\bullet}$ & 0.06              & 90.41\% & \textbf{90.15\%}$^{\bullet}$ & 0.03                \\
German     & 68.37\% & 76.60\% & \textbf{73.40\%}$^{\bullet}$ & 0.07              & 76.13\% & \textbf{72.67\%}$^{\bullet}$ & 0.1                 \\
CMC        & 48.31\% & 55.15\% & \textbf{55.05\%}$^{\bullet}$ & 0                 & 55.66\% & \textbf{55.41\%}$^{\bullet}$ & 0                   \\
CONNECT-4  & 71.73\% & 75.38\% & \textbf{76.69\%}$^{\bullet}$ & 0.3               & 77.58\% & \textbf{76.02\%}$^{\bullet}$ & 0.3                 \\
LAND-COVER & 76.55\% & 83.69\% & \textbf{77.59\%}             & 0.54              & 83.80\% & \textbf{77.14\%}             & 0.37                \\
LETTER     & 85.65\% & 91.56\% & \textbf{86.01\%}             & 0.9               & 93.61\% & \textbf{86.15\%}             & 0.9                 \\
ISOLET     & 79.83\% & 93.68\% & \textbf{81.40\%}$^{\bullet}$ & 0.57              & 93.32\% & \textbf{81.77\%}$^{\bullet}$ & 0.33                \\ \hline
\end{tabular}
\end{threeparttable}
\end{center}
\end{table*}

\begin{table*}[ht]
\caption{The number of nodes of different methods. We examine ReDT with two teacher models, including RF and GBDT, dubbed ReDT (RF) and ReDT (GBDT), respectively.}
\label{table3}
\begin{center}
\begin{threeparttable}
\begin{tabular}{l|c|cc|cc}
\hline
Dataset    & DT     & RF      & {\upshape ReDT} (RF)       & GBDT   & {\upshape ReDT} (GBDT)     \\ \hline
ADULT      & 7,869  & 244,832 & \textbf{2,286}$^{\bullet}$ & 1,486  & \textbf{2,023}$^{\bullet}$ \\
CRX        & 103    & 6,191   & \textbf{48}$^{\bullet}$    & 1,336  & \textbf{65}$^{\bullet}$   \\
EEG        & 1,858  & 125,289 & 1,948                      & 1,489  & 1,939                     \\
BANK       & 4,302  & 223,906 & \textbf{1,678}$^{\bullet}$ & 1,470  & \textbf{1,603}$^{\bullet}$ \\
GERMAN     & 227    & 11,063  & \textbf{140}$^{\bullet}$   & 1,404  & \textbf{172}$^{\bullet}$   \\
CMC        & 630    & 16,220  & \textbf{202}$^{\bullet}$   & 4,206  & \textbf{275}$^{\bullet}$   \\
CONNECT-4  & 18,152 & 470,261 & \textbf{8,813}$^{\bullet}$ & 4,426  & \textbf{8,740}$^{\bullet}$ \\
LAND-COVER & 85     & 6,100   & \textbf{43}$^{\bullet}$    & 9,492  & \textbf{49}$^{\bullet}$    \\
LETTER     & 2,752  & 168,101 & \textbf{2464}$^{\bullet}$  & 38,631 & \textbf{2,459}$^{\bullet}$ \\
ISOLET     & 707    & 58,905  & \textbf{464}$^{\bullet}$   & 32,751 & \textbf{593}$^{\bullet}$   \\ \hline
\end{tabular}
\end{threeparttable}
\end{center}
\end{table*}

\subsection{ReDT with Forest-based Teachers}\label{33}
In this section, we compare the performance of ReDT and the corresponding forest-based teacher model. Besides, the performance of the decision trees trained with hard labels is also provided for reference.  Specifically, we evaluate two performance metrics, including test accuracy and the number of nodes. The first one is used to measure the accuracy, while the second evaluates the size of the model. As shown in Table \ref{table2}, the test accuracy of ReDT is better than the one of decision trees on all datasets regardless of which teacher model is used. Specifically, ReDT has achieved an increase of almost $5\%$ accuracy compared to DT on half of the datasets, and the improvement is significant on seven of those datasets. In particular, on three datasets (BAND, ADULT, and CONNECT-4), ReDT is on par with its teacher model. 

It is interesting to note that the nodes of ReDT are significantly less than those of DT on almost all datasets (except for EEG dataset), as shown in Table \ref{table3}. Not to mention that compared to the bagging-based teacher model ($ i.e. $ RF), the number of ReDT nodes is 2 orders of magnitude less. Even compared to GBDT, ReDT has fewer nodes most of the time. This phenomenon may probably come from two aspects: \textbf{(1)} The information about the latent distribution $\mathbb{P}_{X \times Y}$ contained in the soft label learned from teacher model better directs the splitting of ReDT; \textbf{(2)} The impurity calculated using the soft label instead of the hard one better approximate the ground-truth one and therefore better guides the splitting. The specific reason is further explored in Section \ref{Abalation}. 

Besides, although the value of optimal $\alpha$ is obtained through the grid search and is different across different datasets, it seems to have a direct connection with the number of categories. Specifically, datasets with more categories (such as LAND-COVER, ISOLET, and LETTER) have a larger optimal $\alpha$. It is presumably due to two reasons:  \textbf{(1)} The more categories, the more likely the soft labels will contain more error information; \textbf{(2)} The more categories, the higher the interference caused by error information contained in soft labels. Regardless of the reason, the number of categories of samples can be used to provide the initial intuition of $\alpha$. Overall, ReDT with the forest-based teachers has better performance compared with the standard DT, while preserves its interpretable nature and efficient training process. The specific connection between $\alpha$ and categories will be further studied in our future works.

\begin{table}[ht]
\caption{Comparison among ReDT with DNN-based teachers and other methods. We examine three DNN architectures, including MLP, LeNet-5, and VGG-11.}
\label{mnist}
\begin{center}
\begin{tabular}{|l|c|c|c|}
\hline
            & MLP & {\upshape LeNet-5} & VGG-11 \\ \hline
Accuracy of DNN   & 98.33\%   & 99.42\%       & 99.49\%      \\ \hline
Accuracy of {\upshape ReDT}   & \textbf{88.21\%} & \textbf{88.57\%}$^{\bullet}$ & \textbf{88.53\%}$^{\bullet}$ \\ \hline
Nodes of {\upshape ReDT}  & \textbf{5,361}$^{\bullet}$    & \textbf{5,173}$^{\bullet}$    & \textbf{5,803}$^{\bullet}$    \\  \hline
Accuracy of DT    & \multicolumn{3}{c|}{87.55\%} \\ \hline
Nodes of DT   & \multicolumn{3}{c|}{5,957} \\ \hline
\end{tabular}
\end{center}
\end{table}

\subsection{ReDT with DNN-based Teachers}
In this section, we compare the performance of ReDT and of its corresponding DNN-based teacher model. As shown in Table \ref{mnist}, similar to the scenario with forest-based teachers, ReDT has higher test accuracy and less nodes compared with decision trees, regardless of which DNN architecture is used. The improvement of performance is also significant in most cases. We notice that there is still a gap in accuracy between ReDT and its teacher model. This limitation may be because the tree-types learning method cannot learn the spatial relationships among the raw pixels.

\subsection{Interpretability of ReDT}
Recall that the interpretability indicates that the prediction process of the model is interpretable. Although the soft label and other techniques are introduced, ReDT retains the excellent interpretability of the standard decision trees. This is because the prediction is still made depending on the leaf node to which the input $\bm{x}$ belongs, while the corresponding leaf node is determined by traversing the tree from the root. 

As we mentioned before, the interpretability is critical in many realms, such as finance and medicine. We use one of the finance dataset, the GERMAN credit dataset, to show certain decision rule examples for further demonstration. In Fig. \ref{interpretability}, we visualize two paths in the tree built on GERMAN, which represents the credit of the customer is classified as good or bad respectively. According to the decision path, the one without checking account, other installment plans and the exorbitant credit amount is considered to have good credit. If a person has a very low level of debt, then it is reasonable to conclude that he has good credit. Similarly, it is also reasonable that people with checking accounts and long duration will be judged with bad credit. 

\begin{figure}[ht]
\centering
\includegraphics[width=0.48\textwidth]{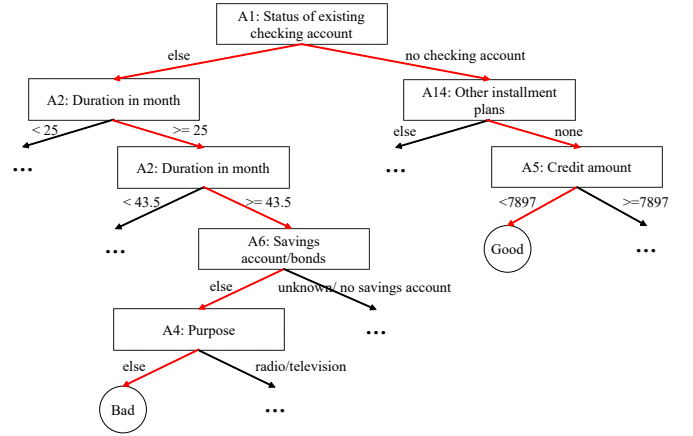}
\caption{Part of the ReDT trained on GERMAN credit dataset. (Two specific decision paths are marked in red.)}
\label{interpretability}
\end{figure}

\begin{table}[ht]
\caption{Comparison of rules between DT and ReDT with different forest-based teachers. $R$ indicates the ratio between the rules of DT and that of ReDT.}
\label{table6}
\begin{center}
\begin{threeparttable}
\small
\scalebox{0.85}{
\begin{tabular}{l|c|cc|cc}
\hline
Dataset    & DT    & {\upshape ReDT (RF)} & $R_1$ & {\upshape ReDT (GBDT)} & $R_2$ \\ \hline
ADULT      & 3,941 & 1,122                & 3.51     & 1,130                  & 3.49        \\
CRX        & 52    & 23                   & 2.26     & 35                     & 1.49        \\
EEG        & 931   & 976                  & 0.95     & 962                    & 0.97        \\
BANK       & 2,153 & 902                  & 2.39     & 871                    & 2.47        \\
GERMAN     & 114   & 68                   & 1.68     & 93                     & 1.23        \\
CMC        & 315   & 98                   & 3.21     & 136                    & 2.32        \\
CONNECT-4  & 9,073 & 4,399                & 2.06     & 4,342                  & 2.09        \\
LAND-COVER & 43    & 21                   & 2.05     & 26                     & 1.65        \\
LETTER     & 1,383 & 1,324                & 1.04     & 1,284                  & 1.08        \\
ISOLET     & 354   & 242                  & 1.46     & 296                    & 1.20        \\ \hline
\end{tabular}
}
\end{threeparttable}
\end{center}
\end{table}

Besides, for the tree-based algorithms, the number of rules ($i.e.$, the decision path corresponding to the leaf node) determines the overall interpretability of the model. The fewer the rules, the easier it is to explain the mechanism behind the model. Without loss of generality, we use random forest and GBDT as the teacher model respectively for further discussion.

As shown in Table \ref{table6}, when the ReDT reaches its best performance through finding the best hyper-parameter $\alpha^{*}$, it still has fewer rules than the standard decision trees in almost all the datasets (except for the EEG dataset). More specifically, in most of the cases, the ratios $R$ between the rules of DT and that of ReDT are greater than 2. In other words, although both decision trees and ReDT are interpretable, ReDT has better overall interpretability.

\subsection{Effects of Hyper-parameter} \label{Abalation}
In this section, we discuss the effects of $\alpha$ in three aspects, including \textbf{(1)} overall interpretability, \textbf{(2)} prediction accuracy, and \textbf{(3)} compression rate. Without loss of generality, we also adopt random forest and GBDT as the teacher model here.

\begin{figure*}[ht]
\centering
\subfigure[RF]{
\label{fig5a}
\includegraphics[width=0.44\textwidth]{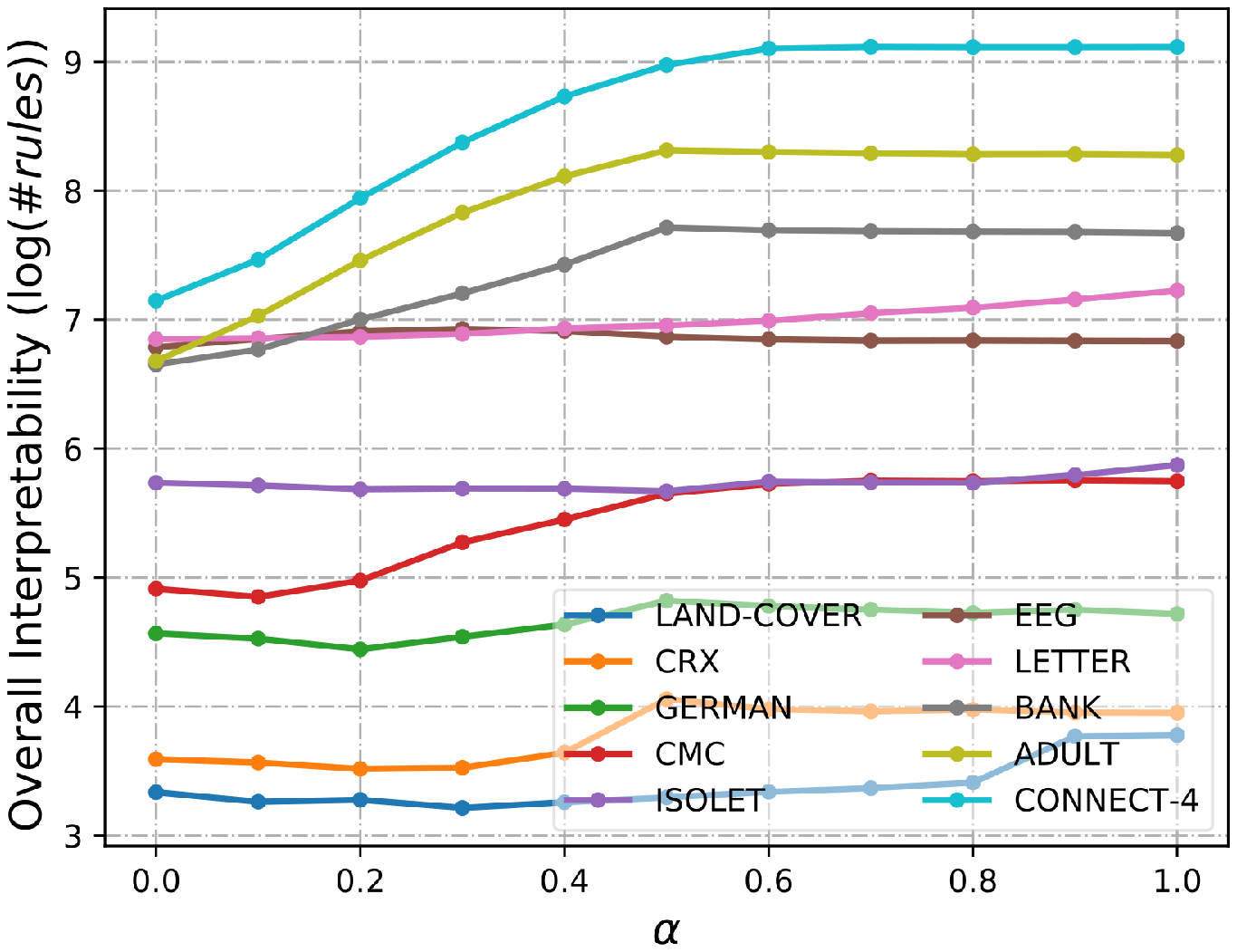}}
\subfigure[GBDT]{
\label{fig5b} %% label for second subfigure
\includegraphics[width=0.44\textwidth]{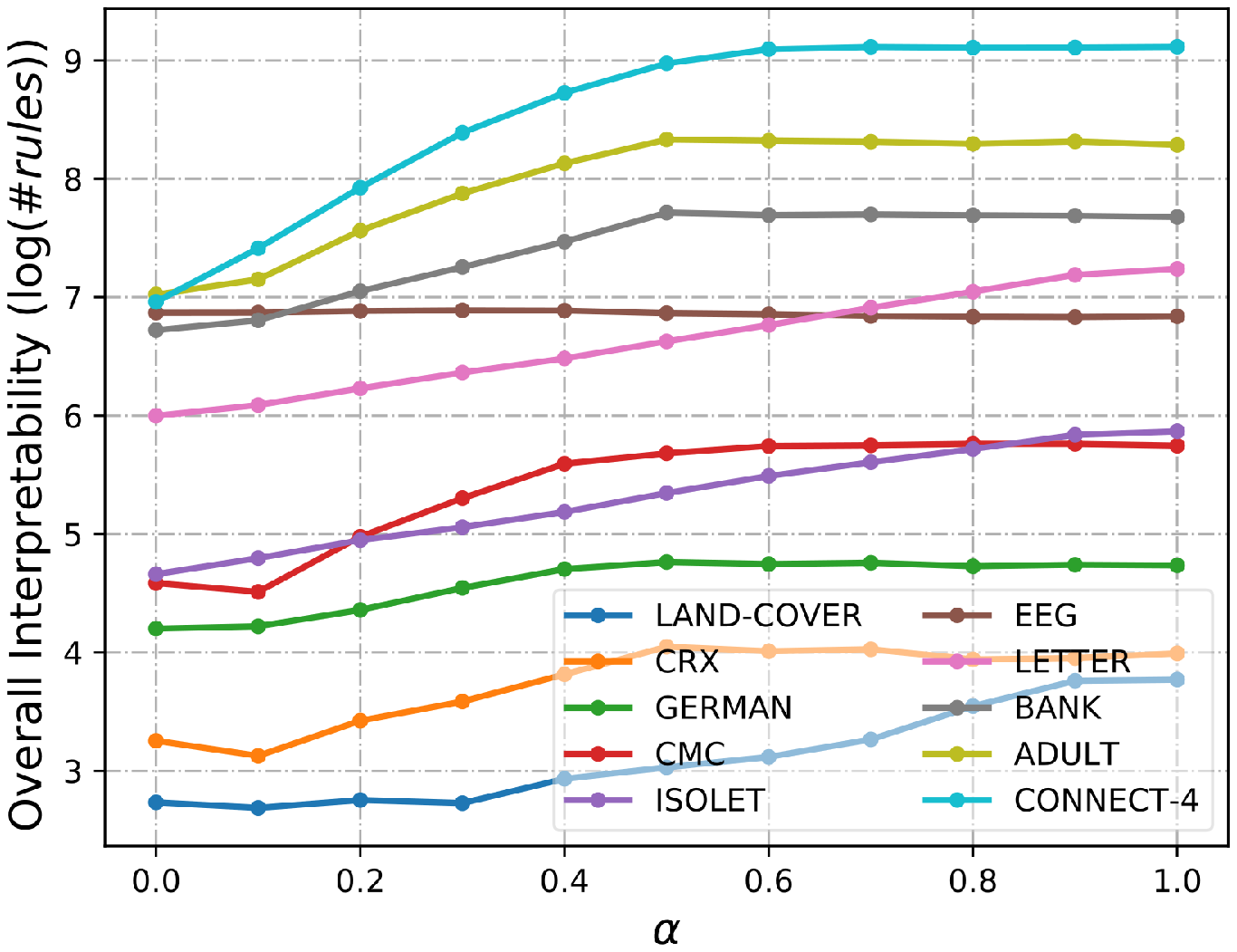}}
\caption{The logarithm number of rules of ReDT with different forest-based teachers under different $\alpha$.}
\label{Rule}
\end{figure*}

\vspace{0.3em}
\noindent \textbf{Overall Interpretability. }
We demonstrate the relation between the number of rules and the hyper-parameter $\alpha$ of ReDT on different datasets. As shown in Fig. \ref{Rule}, the number of rules increases with the hyper-parameter $\alpha$. In other words, the overall interpretability decreases with the increase of hyper-parameter $\alpha$. Recall that the larger the $\alpha$, the smaller the proportion of the soft label in the mixed label used for training ReDT. This phenomenon indicates that the `dark knowledge', which is probably the information about the latent distribution $\mathbb{P}_{X \times Y}$, contained in the soft label better directs the splitting of ReDT, no matter the information is correct or not.

\begin{figure*}[ht]
\centering
\subfigure[RF]{
\label{fig4a}
\includegraphics[width=0.44\textwidth]{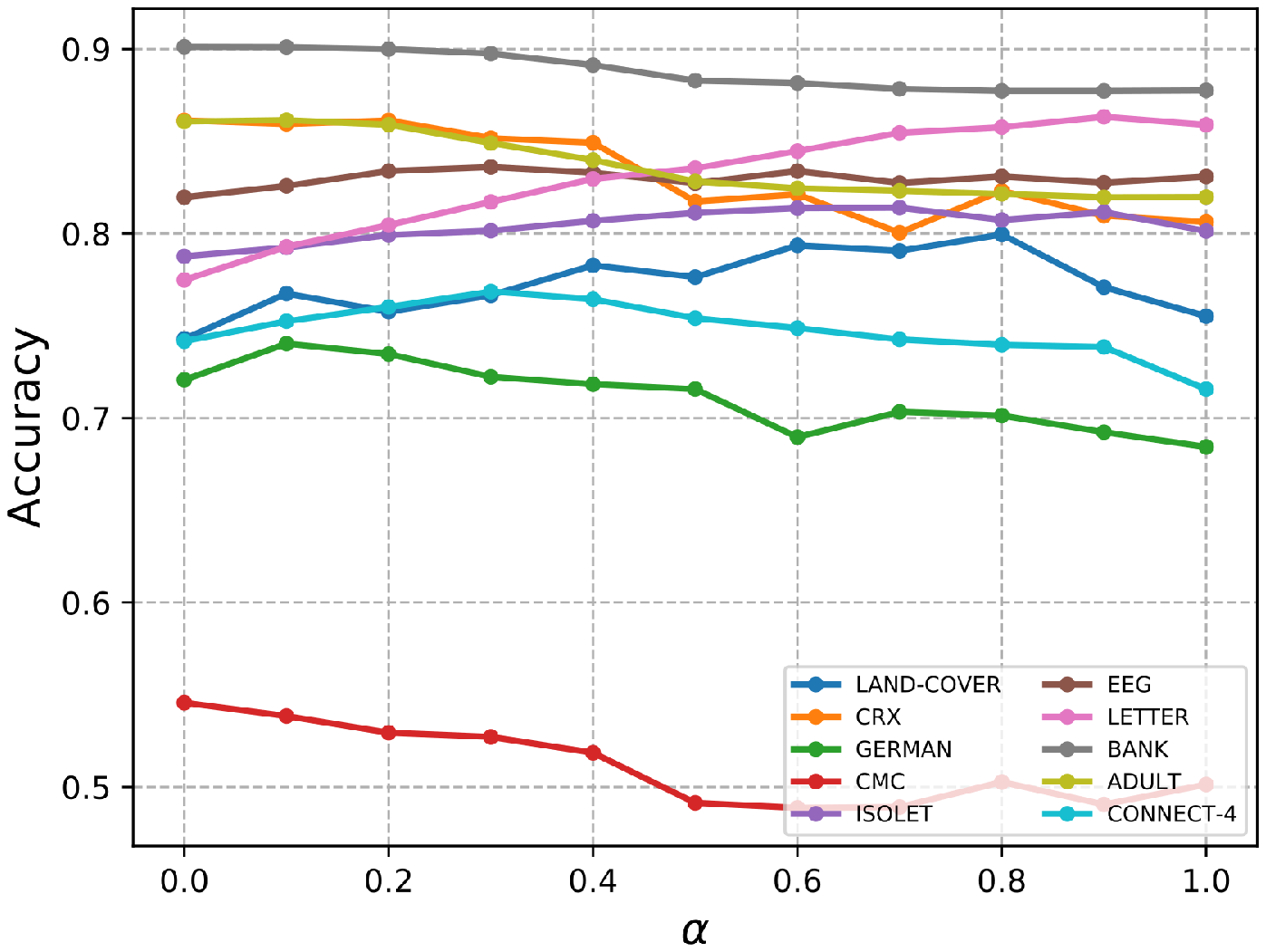}}
\subfigure[GBDT]{
\label{fig4b} %% label for second subfigure
\includegraphics[width=0.44\textwidth]{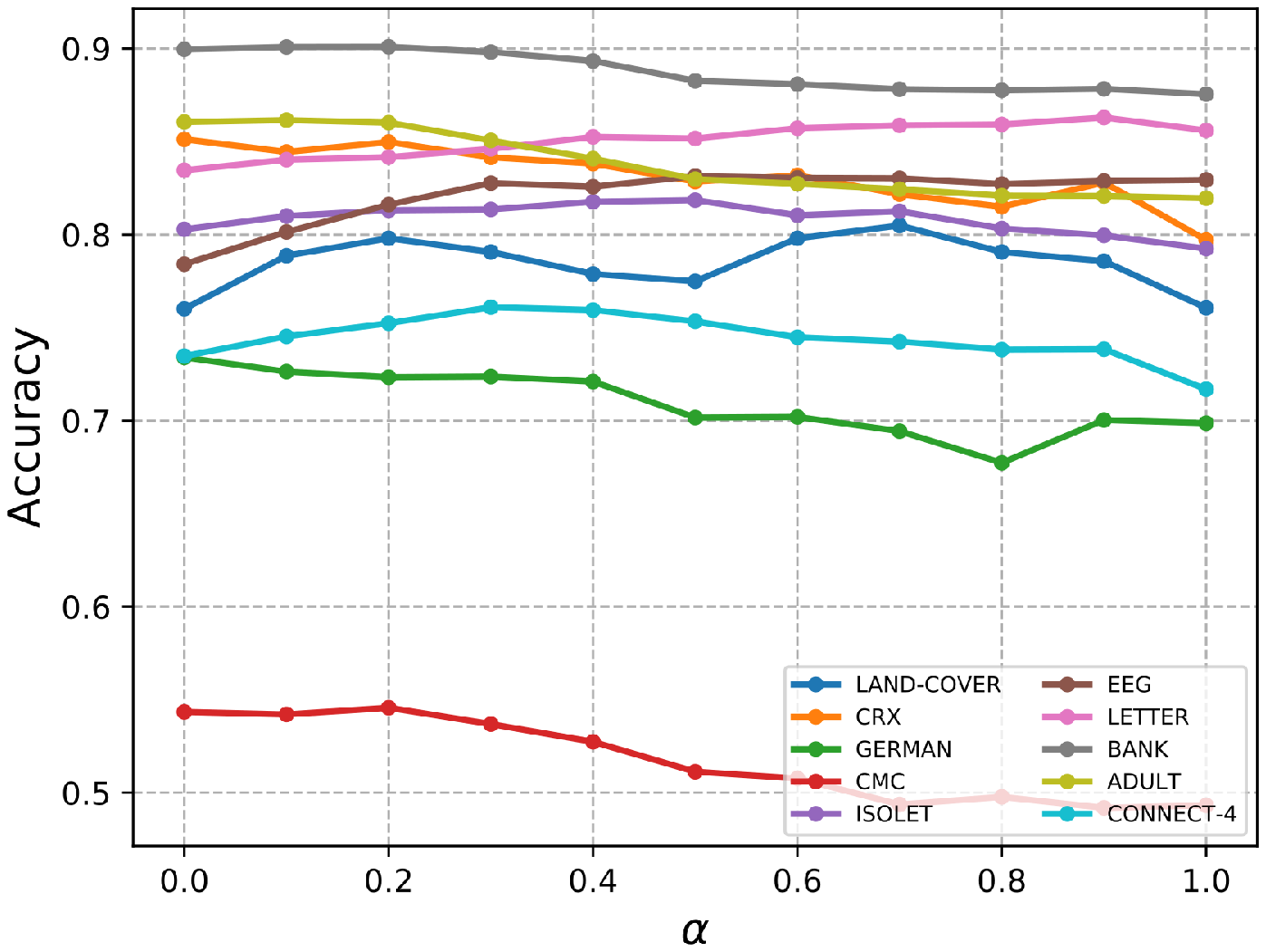}}
\caption{Test accuracy of ReDT with different forest-based teachers under different $\alpha$.}
\label{ACC}
\end{figure*}

\vspace{0.3em}
\noindent \textbf{Prediction Accuracy.}
As shown in Fig. \ref{ACC}, although the optimal value of $\alpha$ depends on the specific characteristics of a dataset, especially the number of categories as stated in Section \ref{33}. When $\alpha=0.2$, the ReDT achieves relatively competitive performance on all datasets. In other words, $\alpha=0.2$ is a good default value to some extent. Besides, since $\alpha$ only affects the training process of the tree and the training of tree is effective, even if the optimal $\alpha$ is obtained from the grid search, the computational consumption is still acceptable when selecting an appropriate step size.

\vspace{0.3em}
\noindent \textbf{Compression Rate.} 
Since ReDT is closely related to the standard DT, DT is a good reference to compare when evaluating the model size of ReDT. The compression rate is defined as the ratio of ReDT's nodes and DT's nodes. The smaller compression rate means the relatively smaller model size. 

As shown in Fig. \ref{Compression}, the compression rates of ReDT are less than 1 regardless of the dataset and $\alpha$ in most cases, which indicates that ReDT has a smaller model size than DT in general. Besides, similar to the circumstance in the overall interpretability, as the hyper-parameter $\alpha$ increases, the compression rate has an upward trend. This is an unexpected gift from the `dark knowledge' distilled soft labels, as demonstrated above.

\begin{figure*}[ht]
\centering
\subfigure[RF]{
\label{figa}
\includegraphics[width=0.43\textwidth]{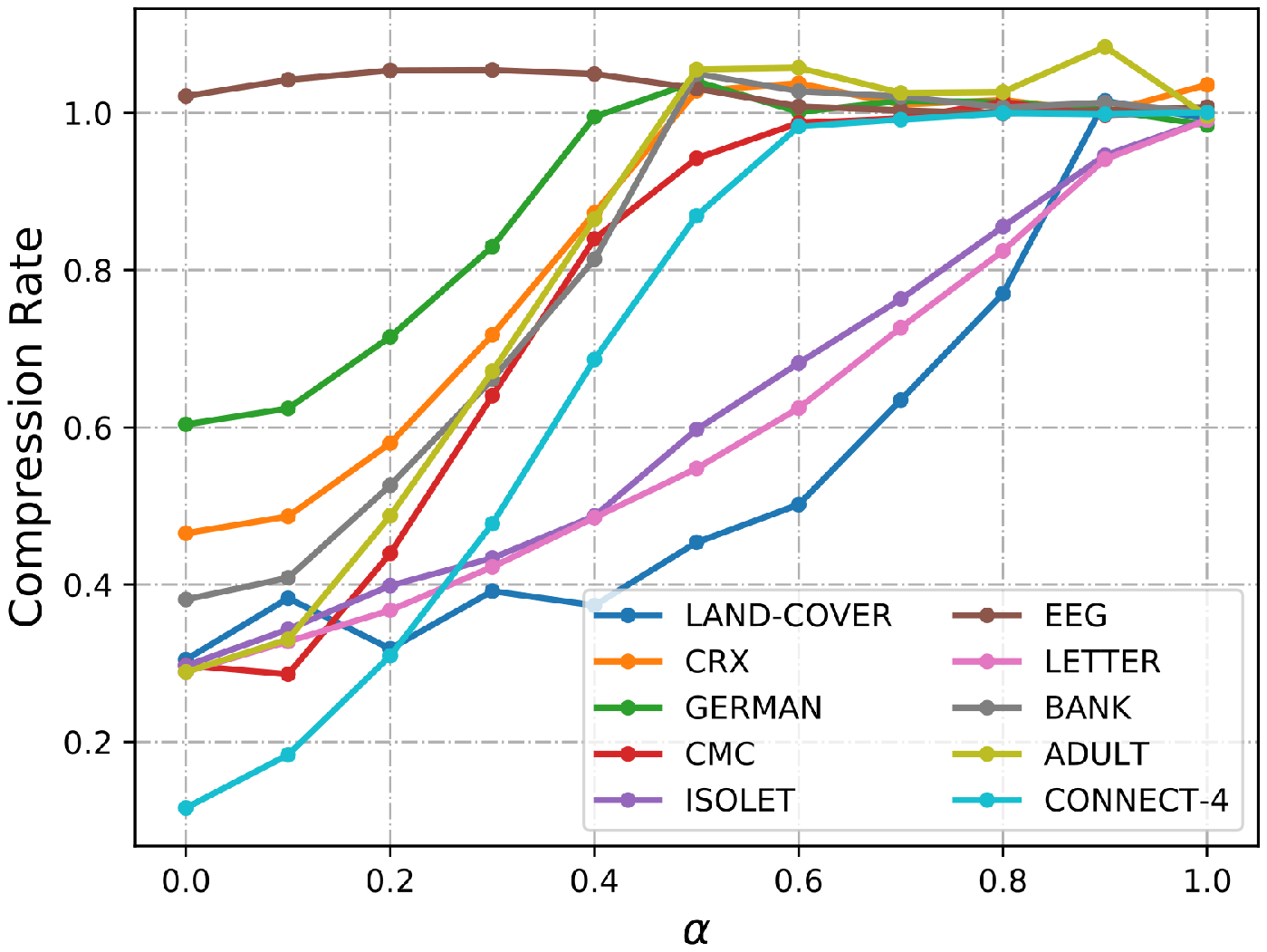}}
\subfigure[GBDT]{
\label{figb} %% label for second subfigure
\includegraphics[width=0.43\textwidth]{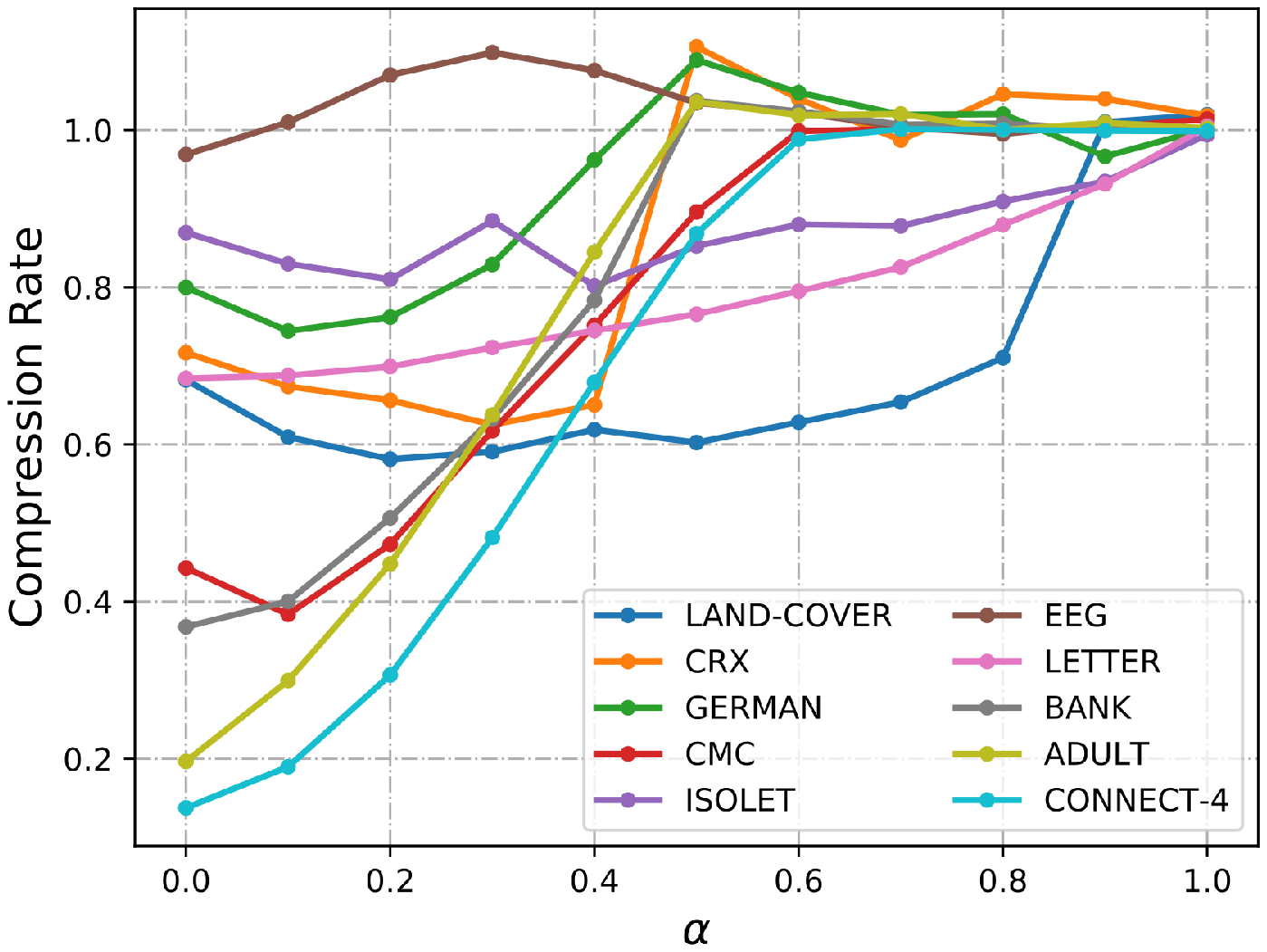}}
\caption{Compression rate of ReDT with different forest-based teachers.}
\label{Compression}
\end{figure*}

\begin{figure*}[ht]
\centering
\subfigure[RF]{
\label{fig2a}
\includegraphics[width=0.43\textwidth]{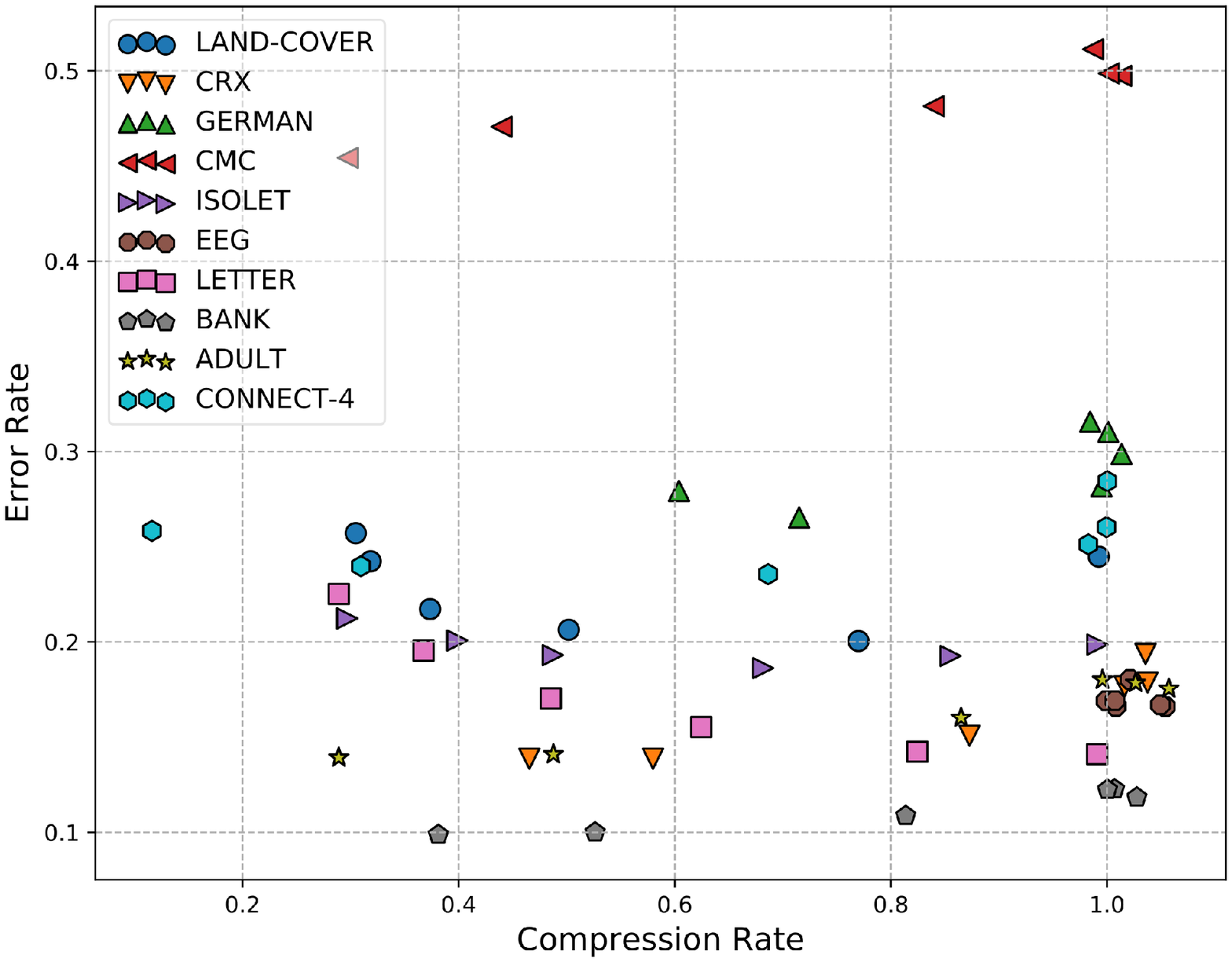}}
\subfigure[GBDT]{
\label{fig2b} %% label for second subfigure
\includegraphics[width=0.43\textwidth]{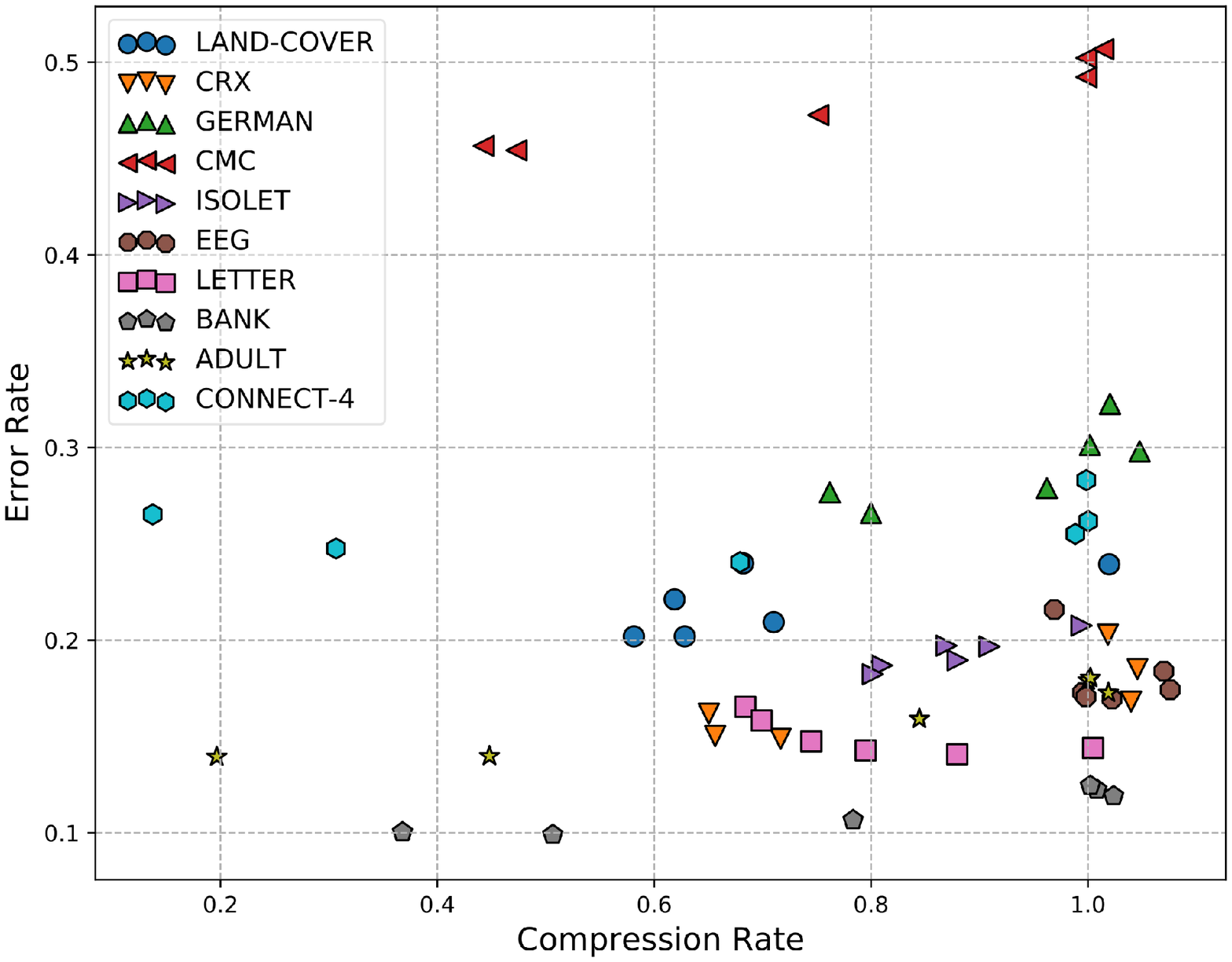}}
\caption{The relationship between error rate and compression rate of ReDT with different forest-based teachers.}
\label{ErrorCompression}
\end{figure*}

\newpage
\subsection{The Relationship between Accuracy and Compression Rate}
In this section, we discuss whether it is possible to obtain a ReDT that has a small model size and high accuracy simultaneously. 

We visualize the relationship between the error rate and the compression rate of ReDT with different forest-based teachers. As shown in Fig. \ref{ErrorCompression}, regardless of the teacher model, all datasets have points at the bottom left of the graph. This phenomenon implies that we can obtain a ReDT with the small model size and relatively high accuracy by adjusting the hyper-parameter $\alpha$. In other words, there is no trade-off between the model size and the accuracy of ReDT to some extent.

\section{Conclusion}
In this paper, we propose a knowledge distillation based decision trees extension, dubbed rectified decision trees (ReDT), to explore the possibility of fulfilling interpretability and effectiveness simultaneously.
Specifically, we extend the splitting criteria and the ending condition of the standard decision trees, which allows training with soft labels distilled from a well-trained teacher model while preserving the deterministic splitting paths.
In particular, we propose a jackknife-based distillation to obtain soft labels.
In contrast to traditional knowledge distillation approaches, backpropagation is not necessarily required for student model in the proposed distillation method. 
The effectiveness of adopting soft labels instead of hard ones is also analyzed empirically and theoretically. 
Besides, extensive experiments also indicate that trained ReDT has a smaller model size than standard decision trees from the aspect of the total nodes and the total rules, which is an unexpected gift from the information distilled from the teacher model.
The proposed method may provide a new angel towards interpretable machine learning.

\section*{Acknowledgments}
This work is supported partly by the National Natural Science Foundation of China under Grant 61771273, the R\&D Program of Shenzhen (JCYJ20180508152204044), the research fund of PCL Future Regional Network Facilities for Large-scale Experiments and Applications (PCL2018KP001). We also thank the support by the Natural Science Foundation of Zhejiang Province under Grant LSY19A010002.

\newpage

\bibliographystyle{IEEEtran}
\bibliography{ref}

\begin{thebibliography}{10}
\providecommand{\url}[1]{#1}
\csname url@samestyle\endcsname
\providecommand{\newblock}{\relax}
\providecommand{\bibinfo}[2]{#2}
\providecommand{\BIBentrySTDinterwordspacing}{\spaceskip=0pt\relax}
\providecommand{\BIBentryALTinterwordstretchfactor}{4}
\providecommand{\BIBentryALTinterwordspacing}{\spaceskip=\fontdimen2\font plus
\BIBentryALTinterwordstretchfactor\fontdimen3\font minus
  \fontdimen4\font\relax}
\providecommand{\BIBforeignlanguage}[2]{{%
\expandafter\ifx\csname l@#1\endcsname\relax
\typeout{** WARNING: IEEEtran.bst: No hyphenation pattern has been}%
\typeout{** loaded for the language `#1'. Using the pattern for}%
\typeout{** the default language instead.}%
\else
\language=\csname l@#1\endcsname
\fi
#2}}
\providecommand{\BIBdecl}{\relax}
\BIBdecl

\bibitem{galton1886regression}
F.~Galton, ``Regression towards mediocrity in hereditary stature.'' \emph{The
  Journal of the Anthropological Institute of Great Britain and Ireland},
  vol.~15, pp. 246--263, 1886.

\bibitem{hosmer2013applied}
D.~W. Hosmer~Jr, S.~Lemeshow, and R.~X. Sturdivant, \emph{Applied logistic
  regression}.\hskip 1em plus 0.5em minus 0.4em\relax John Wiley \& Sons, 2013,
  vol. 398.

\bibitem{breiman1984classification}
L.~Breiman, ``Classification and regression trees,'' \emph{The Wadsworth \&
  Brooks/Cole}, 1984.

\bibitem{grbovic2018real}
M.~Grbovic and H.~Cheng, ``Real-time personalization using embeddings for
  search ranking at airbnb,'' in \emph{SIGKDD}, 2018.

\bibitem{cui2018detection}
Z.~Cui, F.~Xue, X.~Cai, Y.~Cao, G.-g. Wang, and J.~Chen, ``Detection of
  malicious code variants based on deep learning,'' \emph{IEEE Transactions on
  Industrial Informatics}, vol.~14, no.~7, pp. 3187--3196, 2018.

\bibitem{li2018deep}
L.~Li, K.~Ota, and M.~Dong, ``Deep learning for smart industry: Efficient
  manufacture inspection system with fog computing,'' \emph{IEEE Transactions
  on Industrial Informatics}, vol.~14, no.~10, pp. 4665--4673, 2018.

\bibitem{breiman2001}
L.~Breiman, ``Random forests,'' \emph{Machine learning}, vol.~45, no.~1, pp.
  5--32, 2001.

\bibitem{friedman2001}
J.~H. Friedman, ``Greedy function approximation: A gradient boosting machine,''
  \emph{Annals of statistics}, pp. 1189--1232, 2001.

\bibitem{breiman1996}
L.~Breiman and N.~Shang, ``Born again trees,'' \emph{University of California,
  Berkeley, Berkeley, CA, Technical Report}, 1996.

\bibitem{meinshausen2010}
N.~Meinshausen, ``Node harvest,'' \emph{The Annals of Applied Statistics}, pp.
  2049--2072, 2010.

\bibitem{hara2018}
S.~Hara and K.~Hayashi, ``Making tree ensembles interpretable: A bayesian model
  selection approach,'' in \emph{ICAIS}, 2018.

\bibitem{zeiler2014}
M.~D. Zeiler and R.~Fergus, ``Visualizing and understanding convolutional
  networks,'' in \emph{ECCV}, 2014.

\bibitem{zhou2018}
B.~Zhou, D.~Bau, A.~Oliva, and A.~Torralba, ``Interpreting deep visual
  representations via network dissection,'' \emph{IEEE transactions on pattern
  analysis and machine intelligence}, 2018.

\bibitem{yosinski2014}
J.~Yosinski, J.~Clune, Y.~Bengio, and H.~Lipson, ``How transferable are
  features in deep neural networks?'' in \emph{NeurIPS}, 2014, pp. 3320--3328.

\bibitem{zhang2018}
Q.~Zhang, W.~Wang, and S.-C. Zhu, ``Examining {CNN} representations with
  respect to dataset bias,'' in \emph{AAAI}, 2018.

\bibitem{chen2016b}
X.~Chen, Y.~Duan, R.~Houthooft, J.~Schulman, I.~Sutskever, and P.~Abbeel,
  ``Info{GAN}: Interpretable representation learning by information maximizing
  generative adversarial nets,'' in \emph{NeurIPS}, 2016.

\bibitem{zhang2018b}
Q.~Zhang, Y.~Nian~Wu, and S.-C. Zhu, ``Interpretable convolutional neural
  networks,'' in \emph{CVPR}, 2018, pp. 8827--8836.

\bibitem{hinton2015}
G.~Hinton, O.~Vinyals, and J.~Dean, ``Distilling the knowledge in a neural
  network,'' in \emph{NeurIPS Workshop}, 2015.

\bibitem{rumelhart1985learning}
D.~E. Rumelhart, G.~E. Hinton, and R.~J. Williams, ``Learning internal
  representations by error propagation,'' California Univ San Diego La Jolla
  Inst for Cognitive Science, Tech. Rep., 1985.

\bibitem{frosst2017}
N.~Frosst and G.~Hinton, ``Distilling a neural network into a soft decision
  tree,'' in \emph{AI$\ast$IA Workshop}, 2017.

\bibitem{yim2017}
J.~Yim, D.~Joo, J.~Bae, and J.~Kim, ``A gift from knowledge distillation: Fast
  optimization, network minimization and transfer learning,'' in \emph{CVPR},
  2017.

\bibitem{liu2019structured}
Y.~Liu, K.~Chen, C.~Liu, Z.~Qin, Z.~Luo, and J.~Wang, ``Structured knowledge
  distillation for semantic segmentation,'' in \emph{CVPR}, 2019.

\bibitem{irsoy2012}
O.~Irsoy, O.~T. Y{\i}ld{\i}z, and E.~Alpayd{\i}n, ``Soft decision trees,'' in
  \emph{ICPR}, 2012.

\bibitem{UCI}
\BIBentryALTinterwordspacing
A.~Asuncion and D.~Newman, ``{UCI} machine learning repository,'' University of
  California, Irvine, School of Information and Computer Sciences, 2017.
  [Online]. Available: \url{http://archive.ics.uci.edu/ml}
\BIBentrySTDinterwordspacing

\bibitem{lecun1998}
Y.~LeCun, L.~Bottou, Y.~Bengio, and P.~Haffner, ``Gradient-based learning
  applied to document recognition,'' \emph{Proceedings of the IEEE}, vol.~86,
  no.~11, pp. 2278--2324, 1998.

\bibitem{Simonyan2014}
K.~Simonyan and A.~Zisserman, ``Very deep convolutional networks for
  large-scale image recognition,'' in \emph{ICLR}, 2015.

\bibitem{pedregosa2011}
F.~Pedregosa, G.~Varoquaux, A.~Gramfort, V.~Michel, B.~Thirion, O.~Grisel,
  M.~Blondel, P.~Prettenhofer, R.~Weiss, V.~Dubourg \emph{et~al.},
  ``Scikit-learn: Machine learning in python,'' \emph{Journal of machine
  learning research}, vol.~12, no. Oct, pp. 2825--2830, 2011.

\bibitem{paszke2019pytorch}
A.~Paszke, S.~Gross, F.~Massa, A.~Lerer, J.~Bradbury, G.~Chanan, T.~Killeen,
  Z.~Lin, N.~Gimelshein, L.~Antiga \emph{et~al.}, ``Pytorch: An imperative
  style, high-performance deep learning library,'' in \emph{NeurIPS}, 2019.

\bibitem{demvsar2006}
J.~Dem{\v{s}}ar, ``Statistical comparisons of classifiers over multiple data
  sets,'' \emph{Journal of Machine learning research}, vol.~7, no. Jan, pp.
  1--30, 2006.

\end{thebibliography}

\end{document}